\documentclass{article} 
\usepackage{nips15submit_e,times}
\usepackage{hyperref}
\usepackage{url}
\usepackage{multirow}

\title{Accelerated Stochastic Gradient Descent for Minimizing Finite Sums}

\author{
Atsushi Nitanda \\
NTT DATA Mathematical Systems Inc.\\
Tokyo, Japan\\
\texttt{nitanda@msi.co.jp} \\
}

\usepackage{cite}
\usepackage{amsmath}
\usepackage{amssymb}
\usepackage{amsthm}
\usepackage{amsfonts}
\usepackage[dvips]{graphicx}

\newtheorem{theorem}{Theorem}
\newtheorem{lemma}{Lemma}

\newtheorem{assumption}{Assumption}
\newtheorem*{lemma_appendix}{Lemma A}

\nipsfinalcopy 

\begin{document}

\maketitle
\begin{abstract}
We propose an optimization method for minimizing the finite sums of smooth convex functions.
Our method incorporates an accelerated gradient descent (AGD) and a stochastic variance reduction gradient (SVRG) in a mini-batch setting.
Unlike SVRG, our method can be directly applied to non-strongly and strongly convex problems.
We show that our method achieves a lower overall complexity than the recently proposed methods that supports non-strongly convex problems.
Moreover, this method has a fast rate of convergence for strongly convex problems.
Our experiments show the effectiveness of our method.
\end{abstract}

\section{Introduction}
We consider the minimization problem:
\begin{equation}
\underset{x \in \mathbb{R}^d}{\mathrm{minimize}}\ \ f(x) \overset{\mathrm{def}}{=} \frac{1}{n} \sum_{i=1}^{n} f_i(x), \label{opt_prob}
\end{equation} 
where $f_1,\ldots,f_n$ are smooth convex functions from $\mathbb{R}^d$ to $\mathbb{R}$.
In machine learning, we often encounter optimization problems of this type, i.e., empirical risk minimization. 
For example, given a sequence of training examples $(a_1,b_1),\ldots,(a_n,b_n)$, 
where $a_i \in \mathbb{R}^d$ and $b_i \in \mathbb{R}$. If we set $f_i(x)=\frac{1}{2}(a_i^Tx-b_i)^2$, then we obtain linear regression.
If we set $f_i(x)=\log(1+\exp(-b_ix^Ta_i))$ $(b_i\in \{-1,1\}$), then we obtain logistic regression.
Each $f_i(x)$ may include smooth regularization terms.
In this paper we make the following assumption.

\begin{assumption} \label{assumption_l_smooth}
Each convex function $f_i(x)$ is $L$-smooth, i.e., there exists $L>0$ such that for all $x,y\in \mathbb{R}^d$,
\begin{equation}
\| \nabla f_i(x) - \nabla f_i(y) \| \leq L\| x-y \|. \label{L-smooth} \nonumber
\end{equation}
\end{assumption}

In part of this paper (the latter half of section 4), we also assume that $f(x)$ is $\mu$-strongly convex.

\begin{assumption} \label{assumption_strongly_convex}
f(x) is $\mu$-strongly convex, i.e., there exists $\mu>0$ such that for all $x,y \in \mathbb{R}^d$,
\begin{equation}
f(x) \geq f(y) + (\nabla f(y), x-y) + \frac{\mu}{2}\| x-y \|^2. \nonumber
\end{equation}
Note that it is obvious that $L \geq \mu$. 
\end{assumption}

Several papers recently proposed effective methods (SAG\cite{RSB2012,SRB2013}, SDCA\cite{SZ2012,SZ2013a}, SVRG\cite{JZ2013}, S2GD\cite{KR2013}, Acc-Prox-SDCA\cite{SZ2014}, 
Prox-SVRG\cite{XZ2014}, MISO\cite{M2015}, SAGA\cite{DBL2014}, Acc-Prox-SVRG\cite{Nit2014}, mS2GD\cite{KLR2015}) for solving problem (\ref{opt_prob}). 
These methods attempt to reduce the variance of the stochastic gradient and achieve the linear convergence rates like a deterministic gradient descent when $f(x)$ is strongly convex.
Moreover, because of the computational efficiency of each iteration, the overall complexities (total number of component gradient evaluations to find an $\epsilon$-accurate solution in expectation) 
of these methods are less than those of the deterministic and stochastic gradient descent methods.

An advantage of the SAG and SAGA is that they support non-strongly convex problems.
Although we can apply any of these methods to non-strongly convex functions by adding a slight $L_2$-regularization, this modification increases the difficulty of model selection.
In the non-strongly convex case, the overall complexities of SAG and SAGA are $O((n+L)/\epsilon)$. 
This complexity is less than that of the deterministic gradient descent, which have a complexity of $O(nL/\epsilon)$, and is a trade-off with $O(n\sqrt{L/\epsilon})$ ,
which is the complexity of the AGD.

In this paper we propose a new method that incorporates the AGD and SVRG in a mini-batch setting like Acc-Prox-SVRG \cite{Nit2014}.
The difference between our method and Acc-Prox-SVRG is that our method incorporates \cite{AZO2015}, which is similar to Nesterov's acceleration \cite{Nes2005}, 
whereas Acc-Prox-SVRG incorporates \cite{Nes2004}.
Unlike SVRG and Acc-Prox-SVRG, our method is directly applicable to non-strongly convex problems and achieves an overall complexity of
\begin{equation}
\tilde{O}\left( n+ \min \left\{ \frac{L}{\epsilon}, n\sqrt{ \frac{L}{\epsilon}}\ \right\} \right),  \nonumber
\end{equation}
where the notation $\tilde{O}$ hides constant and logarithmic terms.
This complexity is less than that of SAG, SAGA, and AGD.
Moreover, in the strongly convex case, our method achieves a complexity 
\begin{equation}
\tilde{O}\left( n + \min \left\{ \kappa,\ n\sqrt{\kappa}\ \right\} \right), \nonumber
\end{equation}
\vspace{1mm}
where $\kappa$ is the condition number $L/\mu$. 
This complexity is the same as that of Acc-Prox-SVRG.
Thus, our method converges quickly for non-strongly and strongly convex problems.

In Section 2 and 3, we review the recently proposed accelerated gradient method \cite{AZO2015} and the stochastic variance reduction gradient \cite{JZ2013}.
In Section 4, we describe the general scheme of our method and prove an important lemma that gives us a novel insight for constructing specific algorithms.
Moreover, we derive an algorithm that is applicable to non-strongly and strongly convex problems and show its quickly converging complexity.
Our method is a multi-stage scheme like SVRG, but it can be difficult to decide when we should restart a stage.
Thus, in Section 5, we introduce some heuristics for determining the restarting time.
In Section 6, we present experiments that show the effectiveness of our method.

\section{Accelerated Gradient Descent}
We first introduce some notations. In this section, $\|\cdot\|$ denotes the general norm on $\mathbb{R}^d$. 
Let $d(x): \mathbb{R}^d \rightarrow \mathbb{R}$ be a distance generating function (i.e., 1-strongly convex smooth function with respect to $\|\cdot\|$).
Accordingly, we define the Bregman divergence by
\begin{equation}
V_x(y) = d(y) - \left( d(x) + (\nabla d(x), y-x ) \right), \ \ \forall x, \forall y \in \mathbb{R}^d, \nonumber
\end{equation}
where $(,)$ is the Euclidean inner product.
The accelerated method proposed in \cite{AZO2015} uses a gradient step and mirror descent steps and 
takes a linear combination of these points. That is, 
\begin{eqnarray}
&(Convex\ Combination)&x_{k+1}\leftarrow \tau_kz_k + (1-\tau_k)y_k, \nonumber \\
&(Gradient\ Descent)&y_{k+1}\leftarrow \arg \min_{y\in\mathbb{R}^d} \left\{\ (\nabla f(x_{k+1}),y-x_{k+1}) + \textstyle \frac{L}{2}\|y-x_{k+1}\|^2 \ \right\}, \nonumber \\
&(Mirror\ Descent)&z_{k+1}\leftarrow \arg \min_{z\in\mathbb{R}^d} \left\{\ \alpha_{k+1} (\nabla f(x_{k+1}),z-z_k) + V_{z_k}(z) \ \right\}. \nonumber
\end{eqnarray}
Then, with appropriate parameters, $f(y_k)$ converge to the optimal value as fast as the Nesterov's accelerated methods \cite{Nes2005,Nes2004} for non-strongly convex problems.
Moreover, in the strongly convex case, we obtain the same fast convergence as Nesterov's methods by restarting this entire procedure.

In the rest of the paper, we only consider the Euclidean norm, i.e., $\|\cdot\| = \|\cdot\|_2$.
\section{Stochastic Variance Reduction Gradient}
To ensure the convergence of stochastic gradient descent (SGD), the learning rate must decay to zero 
so that we can reduce the variance effect of the stochastic gradient. This slows down the convergence.
Variance reduction techniques \cite{JZ2013,XZ2014,KR2013,KLR2015} such as SVRG have been proposed to solve this problem.
We review SVRG in a mini-batch setting \cite{Nit2014,KLR2015}.
SVRG is a multi-stage scheme. During each stage, this method performs $m$ SGD iterations using the following direction,
\[ v_k = \nabla f_{I_k}(x_k)-\nabla f_{I_k}(\tilde{x}) + \nabla f(\tilde{x}), \]
where $\tilde{x}$ is a starting point at stage, $k$ is an iteration index, $I_k=\{i_1,\ldots,i_b\}$ is a uniformly randomly chosen size $b$ subset of $\{1,2,\ldots,n\}$, and 
$f_{I_k}=\frac{1}{b}\sum_{j=1}^b f_{i_j}$.
Note that $v_k$ is an unbiased estimator of gradient $\nabla f(x_k)$: $\mathbb{E}_{I_k}\left[v_k\right] = \nabla f(x_k)$, where $\mathbb{E}_{I_k}$ 
denote the expectation with respect to $I_k$.
A bound on the variance of $v_k$ is given in the following lemma, which is proved in the Supplementary Material.
\begin{lemma} \label{svrg_bound_lemma} Suppose Assumption \ref{assumption_l_smooth} holds, and let $x_*=\underset{ x \in \mathbb{R}^d }{\arg \min}f(x)$. 
Conditioned on $x_{k}$, we have
\begin{equation}
\mathbb{E}_{I_k}\| v_k - \nabla f(x_k) \|^2 \leq \displaystyle 4L\frac{n-b}{b(n-1)} \left(f(x_k)-f(x_*)+f(\tilde{x})-f(x_*) \right). \label{svrg_bound}
\end{equation}
\end{lemma}
Due to this lemma, SVRG with $b=1$ achieves a complexity of $O ( ( n + \kappa )\log \frac{1}{\epsilon} )$.

\section{Algorithms}
We now introduce our {\it Accelerated efficient Mini-batch SVRG (AMSVRG)} which incorporates AGD and SVRG in a mini-batch setting. 
Our method is a multi-stage scheme similar to SVRG. During each stage, this method performs several APG-like \cite{AZO2015} iterations and uses SVRG direction in a mini-batch setting. 
Each stage of AMSVRG is described in Figure \ref{each_stage}. 

\begin{figure}[h]
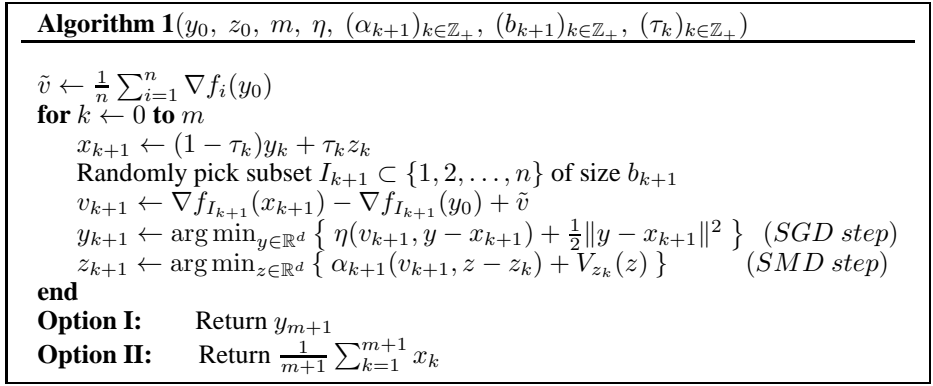

\begin{center}
\fbox{\rule{0cm}{0cm} 
\begin{tabular}{l}
{\bf Algorithm 1}$(y_0,\ z_0,\ m,\ \eta,\ (\alpha_{k+1})_{k\in \mathbb{Z}_+},\ (b_{k+1})_{k\in \mathbb{Z}_+},\ (\tau_k)_{k\in \mathbb{Z}_+} )$ 
\\ \hline \\
$\tilde{v} \leftarrow \frac{1}{n} \sum_{i=1}^n\nabla f_i(y_0)$ \\
{\bf for} $k \leftarrow 0$ {\bf to} $m$ \\
\ \ \ \ \ \ $x_{k+1} \leftarrow (1-\tau_k)y_k + \tau_k z_k$ \\
\ \ \ \ \ \ Randomly pick subset $I_{k+1} \subset \left\{ 1,2,\ldots,n \right\}$ of size $b_{k+1}$ \\
\ \ \ \ \ \ $v_{k+1} \leftarrow \nabla f_{I_{k+1}}(x_{k+1}) - \nabla f_{I_{k+1}}(y_0) + \tilde{v}$ \\
\ \ \ \ \ \ $y_{k+1} \leftarrow \arg \min_{y\in\mathbb{R}^d} \left\{\ \eta(v_{k+1},y-x_{k+1}) + \frac{1}{2}\|y-x_{k+1}\|^2 \ \right\}\ \ (SGD\ step)$\\
\ \ \ \ \ \ $z_{k+1} \leftarrow \arg \min_{z\in\mathbb{R}^d} \left\{\ \alpha_{k+1} (v_{k+1},z-z_k) + V_{z_k}(z) \ \right\}\ \ \ \ \ \ \ \ \ \ \ (SMD\ step)$\\
{\bf end} \\
{\bf Option I: }  
\ \ \ \ \ \ Return $y_{m+1}$ \\
{\bf Option II: } 
\ \ \ \ \ Return $\frac{1}{m+1}\sum_{k=1}^{m+1}x_{k}$ \\
\end{tabular}
\rule{0cm}{0cm}}
\end{center}
\caption{Each stage of AMSVRG}
\label{each_stage}
\end{figure}

\subsection{Convergence analysis of the single stage of AMSVRG}
Before we introduce the multi-stage scheme, we show the convergence of Algorithm 1.
The following lemma is the key to the analysis of our method and gives us an insight on how to construct algorithms.

\begin{lemma} \label{key_lemma}
Consider Algorithm 1 in Figure \ref{each_stage} under Assumption \ref{assumption_l_smooth}. 
We set $\delta_{k} = \frac{n-b_{k}}{b_{k}(n-1)}$.
Let $x_*\in \arg \min_{x \in \mathbb{R}^d} f(x)$. 
If $\eta = \frac{1}{L}$, then we have,
\begin{eqnarray*}
&&\hspace{-15mm}\sum_{k=0}^{m} \alpha_{k+1} \left( \frac{1}{\tau_k} - (1+4\delta_{k+1})L\alpha_{k+1} \right) \mathbb{E}[f(x_{k+1})-f(x_*)]  + L\alpha_{m+1}^2 \mathbb{E}[ f(y_{m+1})  - f(x_*) ]\\ \nonumber
&&\leq V_{z_0}(x_*)  + \sum_{k=1}^{m} \left( \alpha_{k+1} \frac{1-\tau_k}{\tau_k} - L\alpha_k^2 \right)\mathbb{E}[ f(y_k) - f(x_*) ] \\
&&\ \ \ \ + \left( \alpha_1 \frac{1-\tau_0}{\tau_0} + 4L \sum_{k=0}^m \alpha_{k+1}^2 \delta_{k+1} \right)( f(y_0) - f(x_*) ).
\end{eqnarray*}
\end{lemma}

To prove Lemma \ref{key_lemma}, additional lemmas are required, which are proved in the Supplementary Material.

\begin{lemma} \label{gradient_descent_inequality}
(Stochastic Gradient Descent).\ 
Suppose Assumption \ref{assumption_l_smooth} holds, and let $\eta=\frac{1}{L}$.  Conditioned on $x_k$, it follows that for $k \geq 1$,
\begin{equation} 
\mathbb{E}_{I_k}[f(y_k)] \leq f(x_k) - \frac{1}{2L}\|\nabla f(x_k)\|^2 + \frac{1}{2L}\mathbb{E}_{I_k}\|v_k-\nabla f(x_k)\|^2. \label{gd_ineq}
\end{equation}
\end{lemma}

\begin{lemma} \label{mirror_descent_inequality}
(Stochastic Mirror Descent).\ Conditioned on $x_k$, we have that for arbitrary $u \in \mathbb{R}^d$,
\begin{equation}
\alpha_k(\nabla f(x_k),z_{k-1}-u) \leq V_{z_{k-1}}(u) - \mathbb{E}_{I_k} [ V_{z_k}(u) ] + \frac{1}{2}\alpha_k^2\|\nabla f(x_k)\|^2 + \frac{1}{2}\alpha_k^2\mathbb{E}_{I_k}\| v_k - \nabla f(x_k) \|^2. \label{md_ineq}
\end{equation}
\end{lemma}

\begin{proof}[Proof of Lemma \ref{key_lemma}] 
We denote $V_{z_k}(x_*)$ by $V_k$ for simplicity. From Lemma \ref{svrg_bound_lemma}, \ref{gradient_descent_inequality}, and \ref{mirror_descent_inequality} with $u=x_*$, 
\begin{eqnarray*}
&&\!\!\!\!\!\!\!\!\!\!\!\!\alpha_{k+1}  (\nabla f(x_{k+1}), z_k - x_*) \\ 
&&\!\!\!\!\!\!\!\!\!\!\!\!\underset{ (\ref{gd_ineq}, \ref{md_ineq}) }{\leq} V_k - \mathbb{E}_{I_{k+1}}[ V_{k+1} ] +  L \alpha_{k+1}^2 ( f(x_{k+1}) - \mathbb{E}_{I_{k+1}}[ f(y_{k+1}) ]) + \alpha_{k+1}^2 \mathbb{E}_{I_{k+1}} \| v_{k+1} - \nabla f(x_{k+1}) \|^2 \\
&&\!\!\!\!\!\!\!\!\!\!\!\!\underset{ (\ref{svrg_bound}) }{\leq} V_k - \mathbb{E}_{I_{k+1}}[ V_{k+1} ] + L \alpha_{k+1}^2 ( f(x_{k+1}) - \mathbb{E}_{I_{k+1}}[ f(y_{k+1}) ]) \\ 
&&\!\!\!\!\!\!\!\!\!\!\!\!\ \ \ \ + 4L \alpha_{k+1}^2 \delta_{k+1}( f(x_{k+1}) - f(x_*) + f(y_0) - f(x_*) )  \\ 
&&\!\!\!\!\!\!\!\!\!\!\!\!= V_k - \mathbb{E}_{I_{k+1}}[ V_{k+1} ] + ( 1 + 4\delta_{k+1})L \alpha_{k+1}^2(  f(x_{k+1}) - f(x_*) ) - L \alpha_{k+1}^2  \mathbb{E}_{I_{k+1}}[ f(y_{k+1}) - f(x_*) ] \\ 
&&\!\!\!\!\!\!\!\!\!\!\!\!\ \ \ \ + 4L \alpha_{k+1}^2 \delta_{k+1}( f(y_0) - f(x_*) ) . 
\end{eqnarray*}
By taking the expectation with respect to the history of random variables $I_1,I_2\ldots$, we have,
\begin{eqnarray}
\alpha_{k+1} \mathbb{E}[(\nabla f(x_{k+1}), z_k - x_*)] &\leq& \mathbb{E}[ V_k - V_{k+1} ] + ( 1 + 4\delta_{k+1})L \alpha_{k+1}^2  \mathbb{E}[f(x_{k+1}) - f(x_*) ]   \nonumber \\
&&\hspace{-15mm} - L \alpha_{k+1}^2  \mathbb{E}[ f(y_{k+1}) - f(x_*) ] + 4L \alpha_{k+1}^2 \delta_{k+1}( f(y_0) - f(x_*) ), \label{key_lem_exp1}
\end{eqnarray}
and we get
\begin{eqnarray}
\sum_{k=0}^{m}\alpha_{k+1} \mathbb{E}[f(x_{k+1})-f(x_*)] &\leq& \sum_{k=0}^{m}\alpha_{k+1} \mathbb{E}[(\nabla f(x_{k+1}), x_{k+1}- x_*)] \nonumber \\
&&\hspace{-30mm} =\ \sum_{k=0}^{m}\alpha_{k+1} ( \mathbb{E}[(\nabla f(x_{k+1}), x_{k+1}- z_k)] +  \mathbb{E}[(\nabla f(x_{k+1}), z_k- x_*)] ) \nonumber \\
&&\hspace{-30mm} =\ \sum_{k=0}^{m}\alpha_{k+1} \left( \frac{1-\tau_k}{\tau_k}\mathbb{E}[(\nabla f(x_{k+1}), y_k - x_{k+1})] +  \mathbb{E}[(\nabla f(x_{k+1}), z_k- x_*)] \right) \nonumber \\
&&\hspace{-30mm} \leq\ \sum_{k=0}^{m} \left( \alpha_{k+1} \frac{1-\tau_k}{\tau_k}\mathbb{E}[ f(y_k) - f(x_{k+1}) ] +  \alpha_{k+1} \mathbb{E}[(\nabla f(x_{k+1}), z_k- x_*)] \right). \label{key_lem_exp2}
\end{eqnarray}
Using (\ref{key_lem_exp1}), (\ref{key_lem_exp2}), and $V_{z_{k+1}}(x_*) \geq 0$, we have
\begin{eqnarray*}
&&\hspace{-15mm}\sum_{k=0}^{m} \alpha_{k+1} \left( 1 + \frac{1-\tau_k}{\tau_k} - (1+4\delta_{k+1})L\alpha_{k+1} \right) \mathbb{E}[f(x_{k+1})-f(x_*)] \\ \nonumber
&&\leq V_0 + \sum_{k=0}^{m}\alpha_{k+1} \frac{1-\tau_k}{\tau_k}\mathbb{E}[ f(y_k) - f(x_*) ] -L \sum_{k=0}^m \alpha_{k+1}^2  \mathbb{E}[ f(y_{k+1})  - f(x_*) ] \\
&&\ \ \ \ + 4L \sum_{k=0}^m \alpha_{k+1}^2 \delta_{k+1}( f(y_0) - f(x_*) ).
\end{eqnarray*}
This completes the proof of Lemma \ref{key_lemma}.
\end{proof}


From now on we consider Algorithm 1 with option 1 and set 
\begin{equation}
\eta=\frac{1}{L},\ \  \alpha_{k+1} = \frac{1}{4L}(k+2),\ \  \frac{1}{\tau_k} = L \alpha_{k+1} + \frac{1}{2},\ \ for\ \  k=0,1,\ldots. \label{params}
\end{equation}

\begin{theorem} \label{theorem1}
Consider Algorithm 1 with option 1 under Assumption \ref{assumption_l_smooth}. For $p \in \left( 0,\frac{1}{2} \right]$, we choose  $b_{k+1}\in \mathbb{Z}_+$ such that $4L \delta_{k+1}\alpha_{k+1} \leq p$.
Then, we have
\begin{equation}
\mathbb{E}[ f(y_{m+1})  - f(x_*) ] \leq \frac {16L}{(m+2)^2}V_{z_0}(x_*) + \frac{5}{2}p( f(y_0) - f(x_*) ). \nonumber
\end{equation}
Moreover, if $m \geq 4 \sqrt{ \frac{LV_{z_0}(x_*)}{ q( f(y_0) - f(x_*) )} }\ $ for $q > 0$, then it follows
\begin{equation}
\mathbb{E}[ f(y_{m+1})  - f(x_*) ] \leq \left(q+\frac{5}{2}p \right)( f(y_0) - f(x_*) ). \nonumber 
\end{equation}
\end{theorem}
\begin{proof}
Using Lemma \ref{key_lemma} and
\begin{eqnarray*}
&&\tau_0 = 1,\ \ \frac{1}{\tau_k}-(1+4\delta_{k+1})L\alpha_{k+1} \geq 0,\\ 
&&\alpha_{k+1} \frac{1-\tau_k}{\tau_k} - L\alpha_k^2 = L\alpha_{k+1}^2 - \frac{1}{2}\alpha_{k+1} - L\alpha_k^2 = - \frac{1}{16L} < 0,
\end{eqnarray*}
we have
\begin{equation}
 L\alpha_{m+1}^2 \mathbb{E}[ f(y_{m+1})  - f(x_*) ] \leq V_{z_0}(x_*) + 4L \sum_{k=0}^m \alpha_{k+1}^2 \delta_{k+1}( f(y_0) - f(x_*) ). \nonumber
\end{equation}
This proves the theorem because $4L \sum_{k=0}^m \alpha_{k+1}^2 \delta_{k+1} \leq p \sum_{k=0}^m \alpha_{k+1} \leq \frac{5p}{32L}(m+2)^2 $.
\end{proof}

Let $b_{k+1}, m \in \mathbb{Z}_+$ be the minimum values satisfying the assumption of Theorem \ref{theorem1} for $p=q=\epsilon$, i.e.,
$b_{k+1}= \left \lceil \frac{n(k+2)}{\epsilon(n-1)+k+2} \right \rceil$ and $m = \left \lceil 4 \sqrt{ \frac{LV_{z_0}(x_*)}{ \epsilon( f(y_0) - f(x_*) )} } \right \rceil$. 
Then, from Theorem \ref{theorem1}, we have an upper bound on the overall complexity (total number of component gradient evaluations 
to obtain $\epsilon$-accurate solution in expectation):
\begin{equation}
O\left( n + \sum_{k=0}^m b_{k+1} \right) \leq O\left( n + m\frac{nm}{\epsilon n+m} \right) 
=O\left( n + \frac{nL}{\epsilon ^2n + \sqrt{\epsilon L}} \right), \nonumber
\end{equation}
where we used the monotonicity of $b_{k+1}$ with respect to $k$ for the first inequality. 
Note that the notation $O$ also hides $V_{z_0}(x_*)$ and $f(y_0)-f(x_*)$.

\subsection{Multi-Stage Scheme}
In this subsection, we introduce AMSVRG, as described in Figure \ref{multi_stage}.
\begin{figure}[h]
\begin{center}
\fbox{\rule{0cm}{0cm} 
\begin{tabular}{l}
{\bf Algorithm 2}$(w_0,\ (m_s)_{s\in \mathbb{Z}_+},\ \eta,\ (\alpha_{k+1})_{k\in \mathbb{Z}_+},\ (b_{k+1})_{k\in \mathbb{Z}_+},\ (\tau_k)_{k\in \mathbb{Z}_+} )$ 
\\ \hline \\
{\bf for} $s \leftarrow 0,\ 1,\ldots$ \\
\ \ \ \ \ \ $y_0 \leftarrow w_s,\ \ z_0 \leftarrow w_s$\\
\ \ \ \ \ \ $w_{s+1} \leftarrow {\bf Algorithm 1 }( y_0,\ z_0,\ m_s,\ \eta,\ (\alpha_{k+1})_{k\in \mathbb{Z}_+},\ (b_{k+1})_{k\in \mathbb{Z}_+},\ (\tau_k)_{k\in \mathbb{Z}_+})$\\
{\bf end} \\
\end{tabular}
\rule{0cm}{0cm}}
\end{center}
\caption{Accelerated efficient Mini-batch SVRG}
\label{multi_stage}
\end{figure}
We consider the convergence of AMSVRG under the following boundedness assumption which has been used in a 
several papers to analyze incremental and stochastic methods (e.g., \cite{BL2005,GOP2014}).
\begin{assumption} (Boundedness) \label{assumption_boundedness}
There is a compact subset $\Omega \subset \mathbb{R}^d$ such that the sequence $\{w_s\}$ generated by AMSVRG is contained in $\Omega$.
\end{assumption}
Note that, if we change the initialization of $z_0\leftarrow w_s$ to $z_0 \leftarrow z: constant$, the above method with this modification will achieve the same convergence 
for general convex problems without the boundedness assumption (c.f. supplementary materials).
However, for the strongly convex case, this modified version is slower than the above scheme.
Therefore, we consider the version described in Figure \ref{multi_stage}. 

From Theorem \ref{theorem1}, we can see that for small $p$ and $q$ (e.g. $p=1/10,\ q=1/4$), the expected value of the objective function is halved at every stage under the assumptions of Theorem \ref{theorem1}.
Hence, running AMSVRG for $O(\log(1/\epsilon))$ outer iterations achieves an $\epsilon$-accurate solution in expectation.
Here, we consider the complexity at stage $s$ to halve the expected objective value.
Let $b_{k+1}, m_s \in \mathbb{Z}_+$ be the minimum values satisfying the assumption of Theorem \ref{theorem1}, i.e.,
$b_{k+1}= \left \lceil \frac{n(k+2)}{p(n-1)+k+2} \right \rceil$ and $m_s = \left \lceil 4 \sqrt{ \frac{LV_{w_s}(x_*)}{ q( f(w_s) - f(x_*) )} } \right \rceil$. 
If the initial objective gap $f(w_s)-f(x_*)$ in stage $s$ is larger than $\epsilon$, then the complexity at stage is
\begin{eqnarray*}
&&\hspace{20mm} O\left( n + \sum_{k=0}^{m_s} b_{k+1} \right) \leq O\left( n + \frac{nm_s^2}{n+m_s} \right) \\
&&= O \left( n + \frac{nL}{n(f(w_s)-f(x_*)) + \sqrt{(f(w_s)-f(x_*)) L}} \right)
\leq O\left( n + \frac{nL}{\epsilon n + \sqrt{\epsilon L}} \right),
\end{eqnarray*}
where we used the monotonicity of $b_{k+1}$ with respect to $k$ for the first inequality. 
Note that by Assumption \ref{assumption_boundedness}, $\{V_{w_s}(x_*)\}_{s=1,2,\ldots}$ are uniformly bounded and notation $O$ also hides $V_{w_s}(x_*)$.
The above analysis implies the following theorem.

\begin{theorem} \label{theorem_non_sc}
Consider AMSVRG under Assumptions \ref{assumption_l_smooth} and \ref{assumption_boundedness}. 
We set $\eta, \alpha_{k+1},$ and $\tau_k$ as in (\ref{params}). 
Let $b_{k+1}= \left \lceil \frac{n(k+2)}{p(n-1)+k+2} \right \rceil$ and $m_s = \left \lceil 4 \sqrt{ \frac{LV_{w_s}(x_*)}{ q( f(w_s) - f(x_*) )} } \right \rceil$, 
where $p$ and $q$ are small values described above. 
Then, the overall complexity to run AMSVRG for $O(\log(1/\epsilon))$ outer iterations or to obtain an $\epsilon$-accurate solution is
\[ O\left( \left( n + \frac{nL}{\epsilon n + \sqrt{\epsilon L}} \right) \log \left(\frac{1}{\epsilon} \right) \right). \]
\end{theorem}

Next, we consider the strongly convex case. 
We assume that $f$ is a $\mu$-strongly convex function.
In this case, we choose the distance generating function $d(x)=\frac{1}{2}\|x\|^2$, so that the Bregman divergence becomes $V_x(y)=\frac{1}{2}\|x-y\|^2$.
Let the parameters be the same as in Theorem \ref{theorem_non_sc}.
Then, the expected value of the objective function is halved at every stage.
Because $m_s \leq 4\sqrt{\frac{\kappa}{q}}$, where $\kappa$ is the condition number $L/\mu$, the complexity at each stage is
\[ O\left( n + \sum_{k=0}^{m_s} b_{k+1} \right) \leq O\left( n + \frac{nm_s^2}{n+m_s} \right) \leq O\left( n + \frac{n\kappa}{n+\sqrt{\kappa}} \right). \]
Thus, we have the following theorem.

\begin{theorem} \label{theorem_sc}
Consider AMSVRG under Assumptions \ref{assumption_l_smooth} and \ref{assumption_strongly_convex}.
Let parameters $\eta, \alpha_{k+1}, \tau_k, m_s$, and $b_{k+1}$ be the same as those in Theorem \ref{theorem_non_sc}.
Then the overall complexity for obtaining $\epsilon$-accurate solution in expectation is
\[ O\left( \left( n + \frac{n\kappa}{n+\sqrt{\kappa}} \right) \log \left(\frac{1}{\epsilon} \right) \right). \]
\end{theorem}

This complexity is the same as that of Acc-Prox-SVRG.
Note that for the strongly convex case, we do not need the boundedness assumption.

Table \ref{comparison_of_complexities} lists the overall complexities of the AGD, SAG, SVRG, SAGA, Acc-Prox-SVRG, and AMSVRG.
The notation $\tilde{O}$ hides constant and logarithmic terms.
By simple calculations, we see that
\[ \frac{n\kappa}{n+\sqrt{\kappa}}=\frac{1}{2}H(\kappa, n\sqrt{\kappa}\ ),\ \ \ 
\frac{nL}{\epsilon n + \sqrt{\epsilon L}}=\frac{1}{2}H\left( \frac{L}{\epsilon}, n\sqrt{\frac{L}{\epsilon}}\ \right), \] 
where $H(\cdot,\cdot)$ is the harmonic mean whose order is the same as $\min\{\cdot,\cdot\}$.
Thus, as shown in Table \ref{comparison_of_complexities}, 
the complexity of AMSVRG is less than or equal to that of other methods in any situation.
In particular, for non-strongly convex problems, our method potentially outperform the others.


\begin{table}[t]
\caption{Comparison of overall complexity. }
\label{complexity-table}
\begin{center}
\begingroup
\renewcommand{\arraystretch}{2}
\begin{tabular}{c|c|c} \hline  
Convexity &Algorithm &Complexity \\
\hline
\multirow{4}{*}{General convex} &AGD                  &$\tilde{O}\left( n\sqrt{ \frac{L}{\epsilon} } \right)$ \\
                                &SAG, SAGA            &$\tilde{O} \left( \frac{n+L}{\epsilon} \right) $  \\
                                &SVRG, Acc-SVRG       &---          \\
                                &{\bf AMSVRG}    &$\tilde{O}\left( n+ \min \left\{ \frac{L}{\epsilon}, n\sqrt{ \frac{L}{\epsilon}}\ \right\} \right)$ \\
\hline 
\multirow{4}{*}{Strongly convex} &AGD                         &$\tilde{O}\left(n\sqrt{\kappa} \right)$   \\
                                 &SAG                         &$\tilde{O} \left( \max \{ n, \kappa \} \right) $  \\
                                 &SVRG                        &$\tilde{O}\left( n + \kappa \right)$ \\
                                 &Acc-SVRG, {\bf AMSVRG} &$\tilde{O}\left( n + \min \left\{ \kappa,\ n\sqrt{\kappa}\ \right\} \right)$ \\
\hline 
\end{tabular}
\endgroup
\end{center}
\label{comparison_of_complexities}
\end{table}

\section{Restart Scheme}
The parameters of AMSVRG are essentially $\eta, m_s,$ and $b_{k+1}\ (i.e.,\ p)$ because the appropriate values of both $\alpha_{k+1}$ and $\tau_k$ can be expressed by $\eta=1/L$ as in (\ref{params}).
It may be difficult to choose an appropriate $m_s$ which is the restart time for Algorithm 1.
So, we propose heuristics for determining the restart time.

First, we suppose that the number of components $n$ is sufficiently large such that the complexity of our method becomes $O(n)$.
That is, for appropriate $m_s$, $O(n)$ is an upper bound on $\sum_{k=0}^{m_s} b_{k+1}$ (which is the complexity term).
Therefore, we estimate the restart time as the minimum index $m \in \mathbb{Z}_+$ that satisfies $\sum_{k=0}^{m} b_{k+1} \geq n$.
This estimated value is upper bound on $m_s$ (in terms of the order).
In this paper, we call this restart method {\it R1}.

Second, we propose an adaptive restart method using SVRG.
In a strongly convex case, we can easily see that if we restart the AGD for general convex problems every $\sqrt{\kappa}$, then 
the method achieves a linear convergence similar to that for strongly convex problems.
The drawback of this restart method is that the restarting time depends on an unknown parameter $\kappa$, so
several papers \cite{DC2013,GB2014,SBC2014} have proposed effective adaptive restart methods.
Moreover, \cite{GB2014} showed that this technique also performs well for general convex problems.
Inspired by their study, we propose an SVRG-based adaptive restart method called {\it R2}. That is, if
\[ ( v_{k+1}, y_{k+1} -y_k ) > 0, \]
then we return $y_k$ and start the next stage.

Third, we propose the restart method {\it R3}, which is a combination of the above two ideas.
When $\sum_{k=0}^{m} b_{k+1}$ exceeds $10n$, we restart Algorithm 1, and when
\[ ( v_{k+1}, y_{k+1} -y_k ) > 0\ \  \wedge\ \   \sum_{k=0}^{m} b_{k+1} > n,  \]
 we return $y_k$ and restart Algorithm 1.

\section{Numerical Experiments}

In this section, we compare AMSVRG with SVRG and SAGA.
We ran an $L2$-regularized multi-class logistic regularization on {\it mnist} and {\it covtype}
and ran an $L2$-regularized binary-class logistic regularization on {\it rcv1}.
The datasets and their descriptions can be found at the LIBSVM website\footnote{http://www.csie.ntu.edu.tw/~cjlin/libsvmtools/datasets/}.
In these experiments, we vary regularization parameter $\lambda$ in $\{ 0,\ 10^{-7},\ 10^{-6},\ 10^{-5}\}$.
We ran AMSVRG using some values of $\eta$ from $[10^{-2},\ 5\times10]$  and $p$ from $[10^{-1},\ 10]$, and then we chose the best $\eta$ and $p$.

The results are shown in Figure \ref{experiments}.
The horizontal axis is the number of single-component gradient evaluations.
Our methods performed well and outperformed the other methods in some cases.
For mnist and covtype, AMSVRG R1 and R3 converged quickly, and for rcv1, AMSVRG R2 worked very well.
This tendency was more remarkable when the regularization parameter $\lambda$ was small.

Note that the gradient evaluations for the mini-batch can be parallelized \cite{DGSX2012,AD2011,SZ2013b}, so AMSVRG may be further accelerated in a parallel framework such as GPU computing.
\begin{figure}
\begin{center}
\begin{tabular}{c|ccc}
$\lambda$ & \hspace{-2mm} mnist & \hspace{-7mm} covtype & \hspace{-7mm} rcv1 \\ 
&
\hspace{-2mm}  \multirow{8}{*}{ \resizebox{41mm}{!}{ \includegraphics[angle=0]{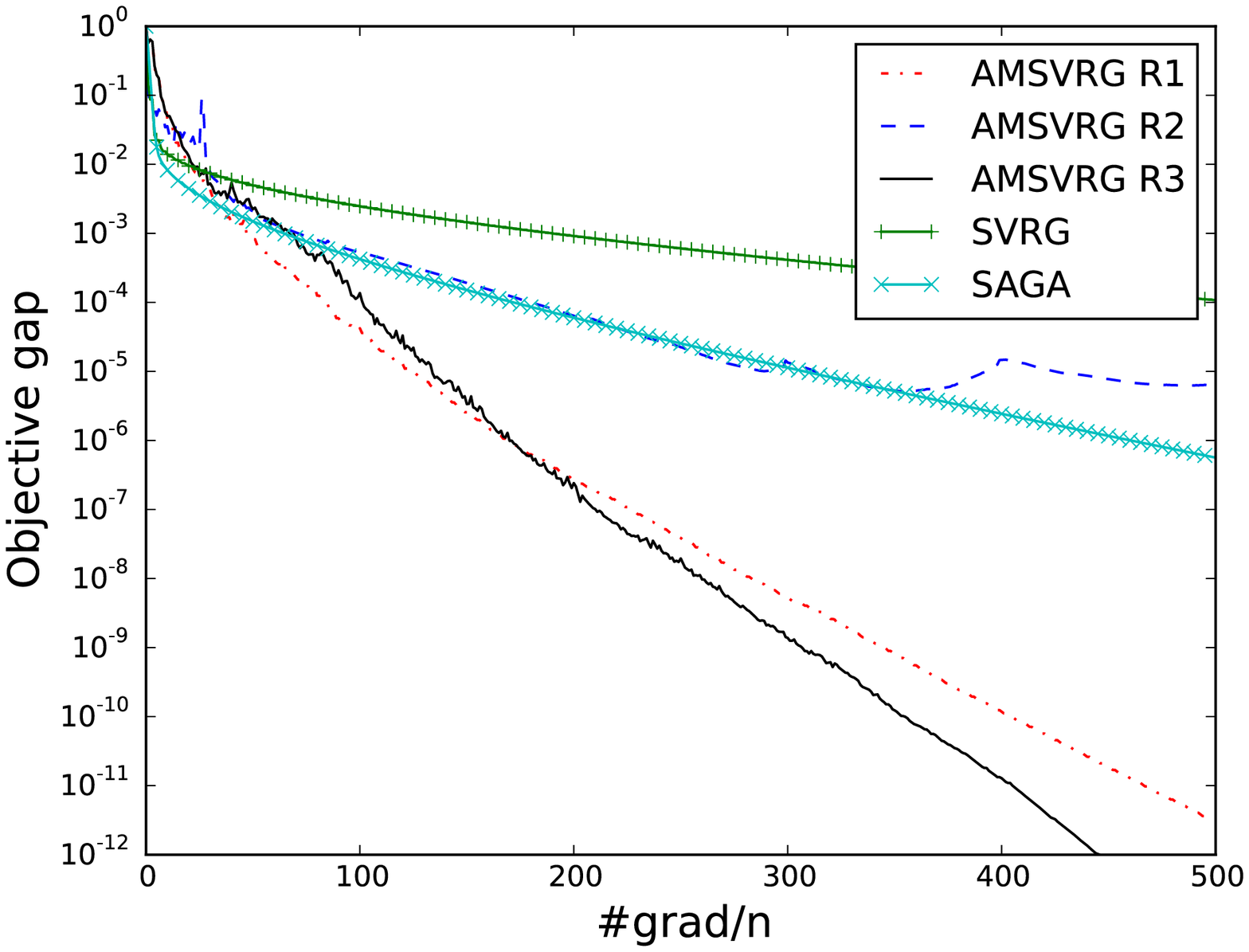} } }&
\hspace{-7mm}  \multirow{8}{*}{ \resizebox{41mm}{!}{ \includegraphics[angle=0]{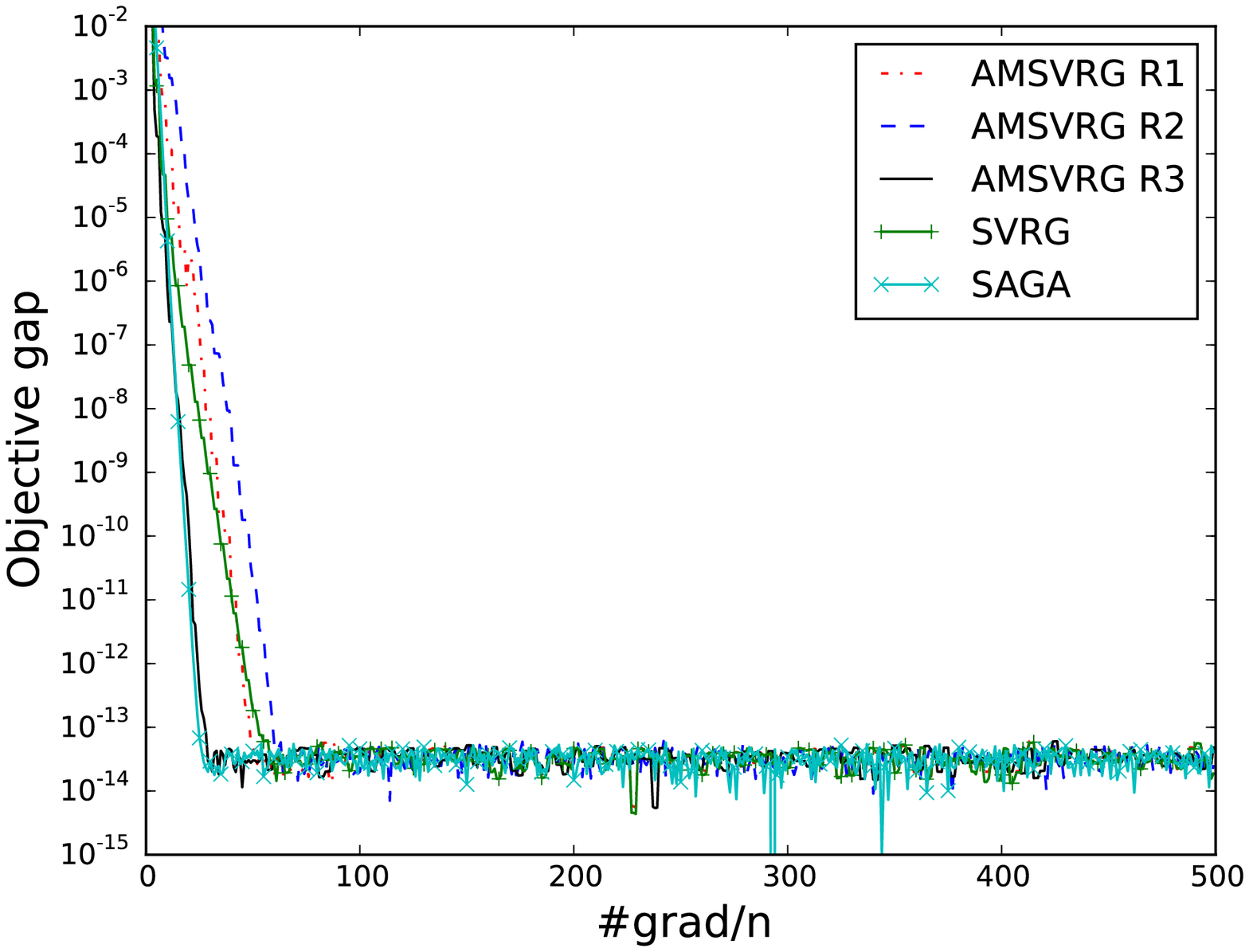} } }&
\hspace{-7mm} \multirow{8}{*}{ \resizebox{41mm}{!}{ \includegraphics[angle=0]{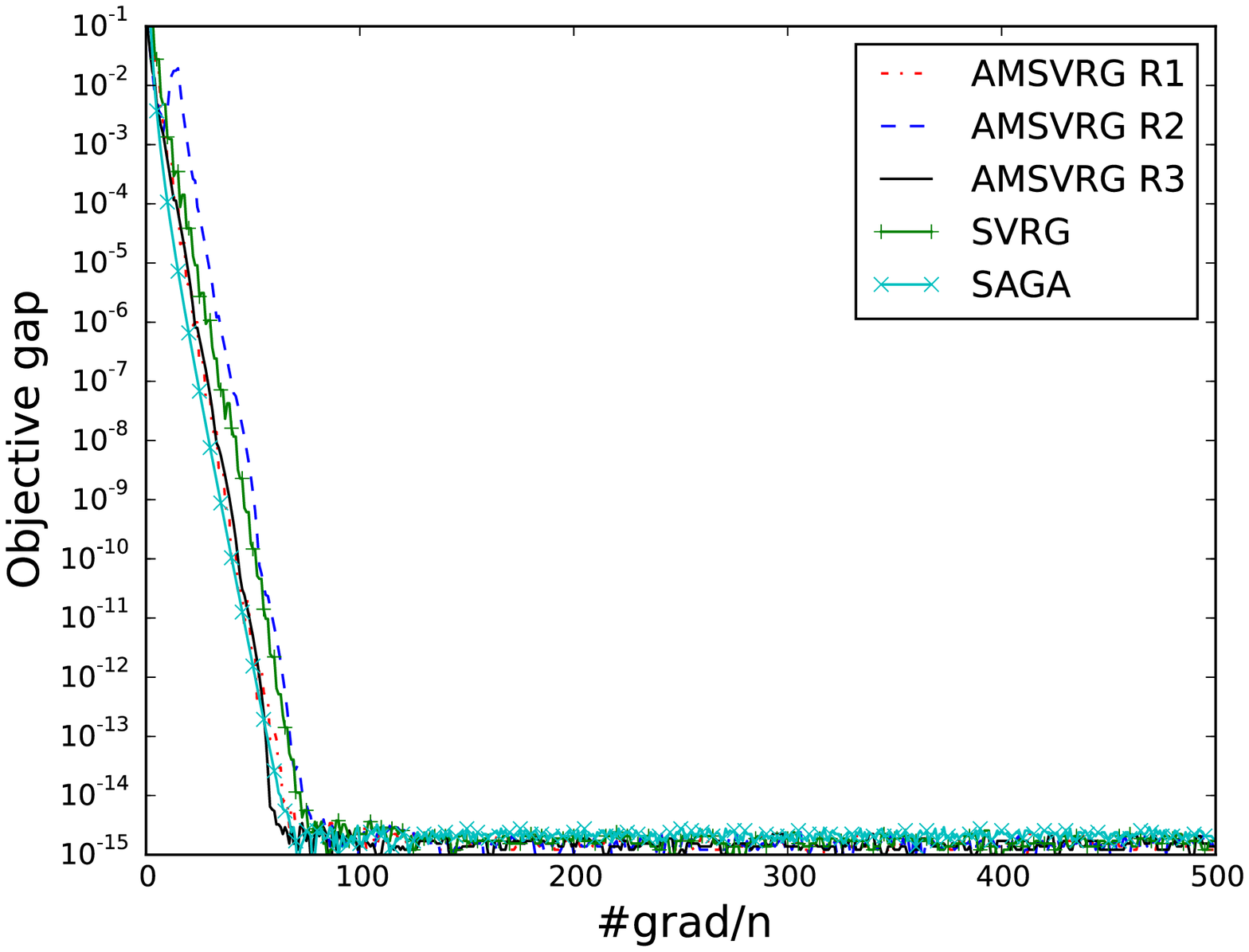} } }\\
&&&\\
&&&\\
$10^{-5}$ &&&\\
&&&\\
&&&\\
&&&\\
&&&\\
&
\hspace{-2mm}  \multirow{8}{*}{ \resizebox{41mm}{!}{ \includegraphics[angle=0]{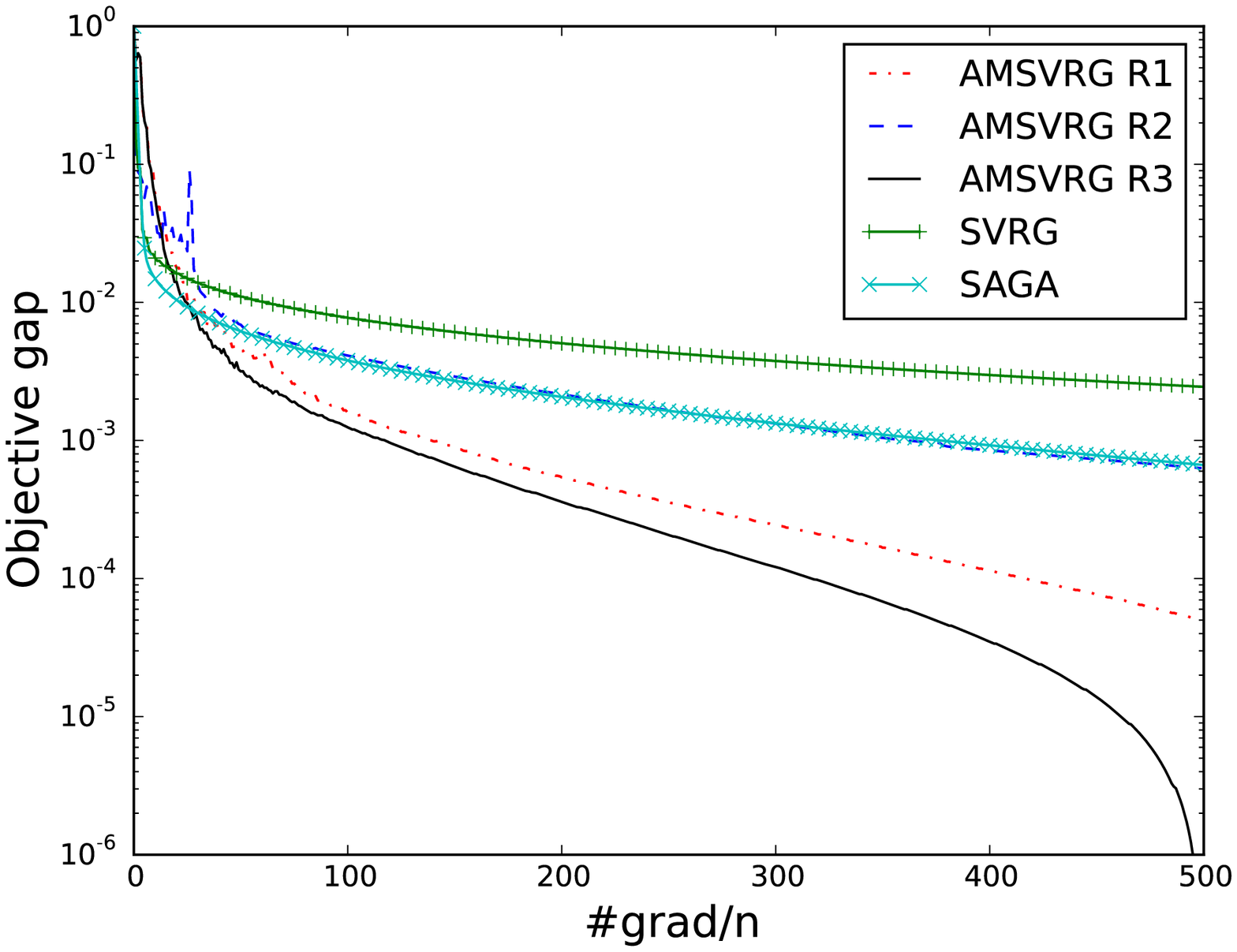} } }&
\hspace{-7mm}  \multirow{8}{*}{ \resizebox{41mm}{!}{ \includegraphics[angle=0]{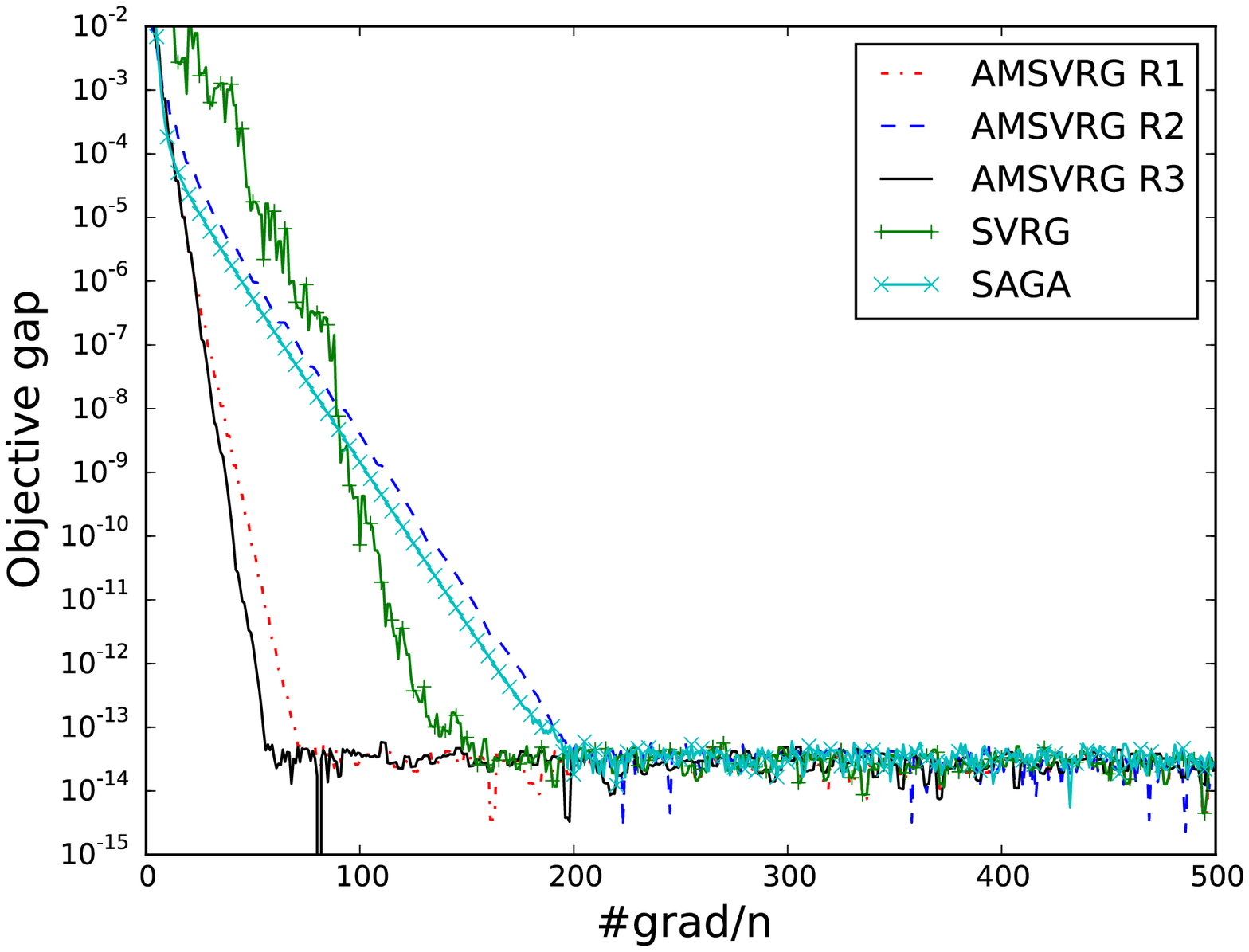} } }&
\hspace{-7mm} \multirow{8}{*}{ \resizebox{41mm}{!}{ \includegraphics[angle=0]{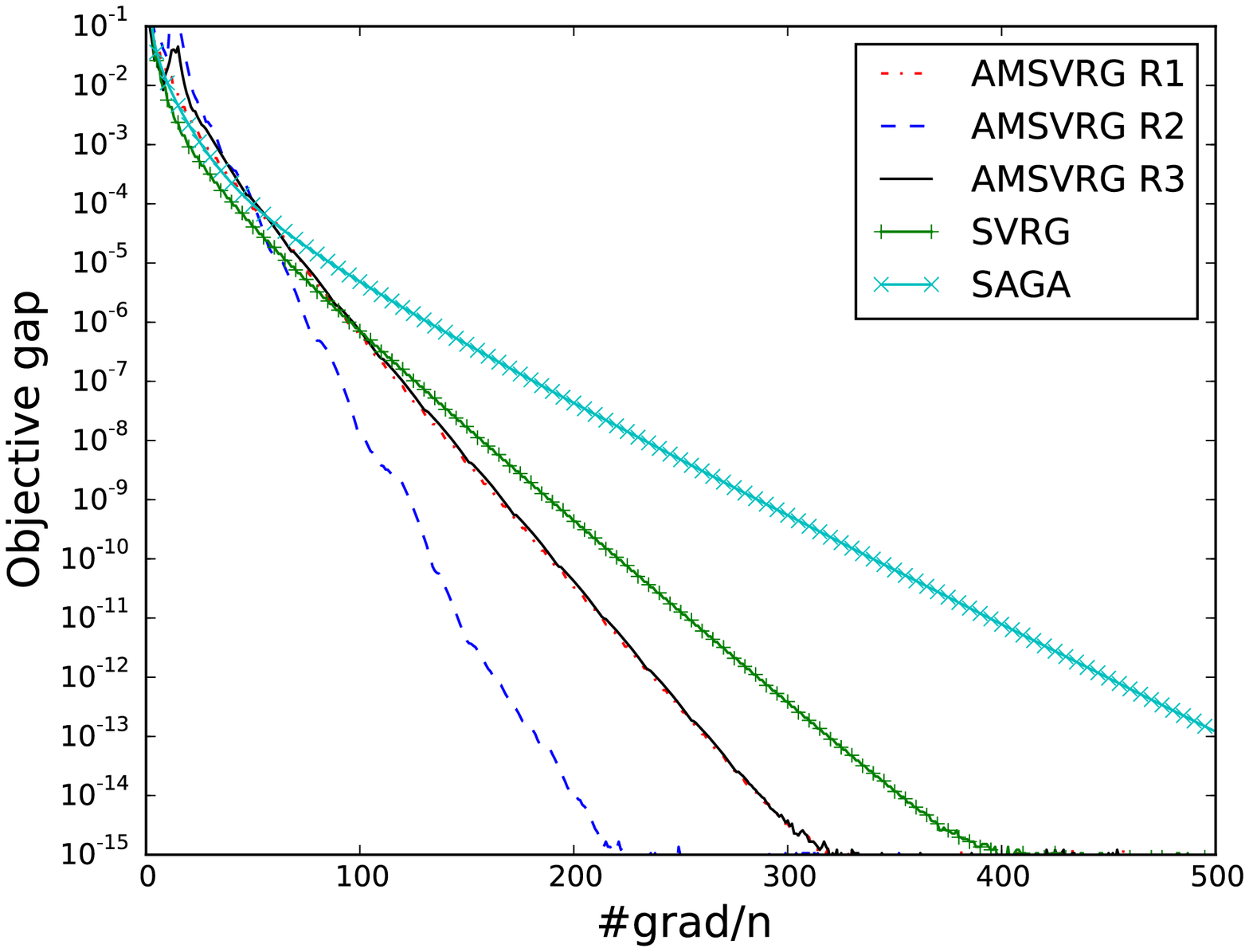} } }\\
&&&\\
&&&\\
$10^{-6}$ &&&\\
&&&\\
&&&\\
&&&\\
&&&\\
&
\hspace{-2mm}  \multirow{8}{*}{ \resizebox{41mm}{!}{ \includegraphics[angle=0]{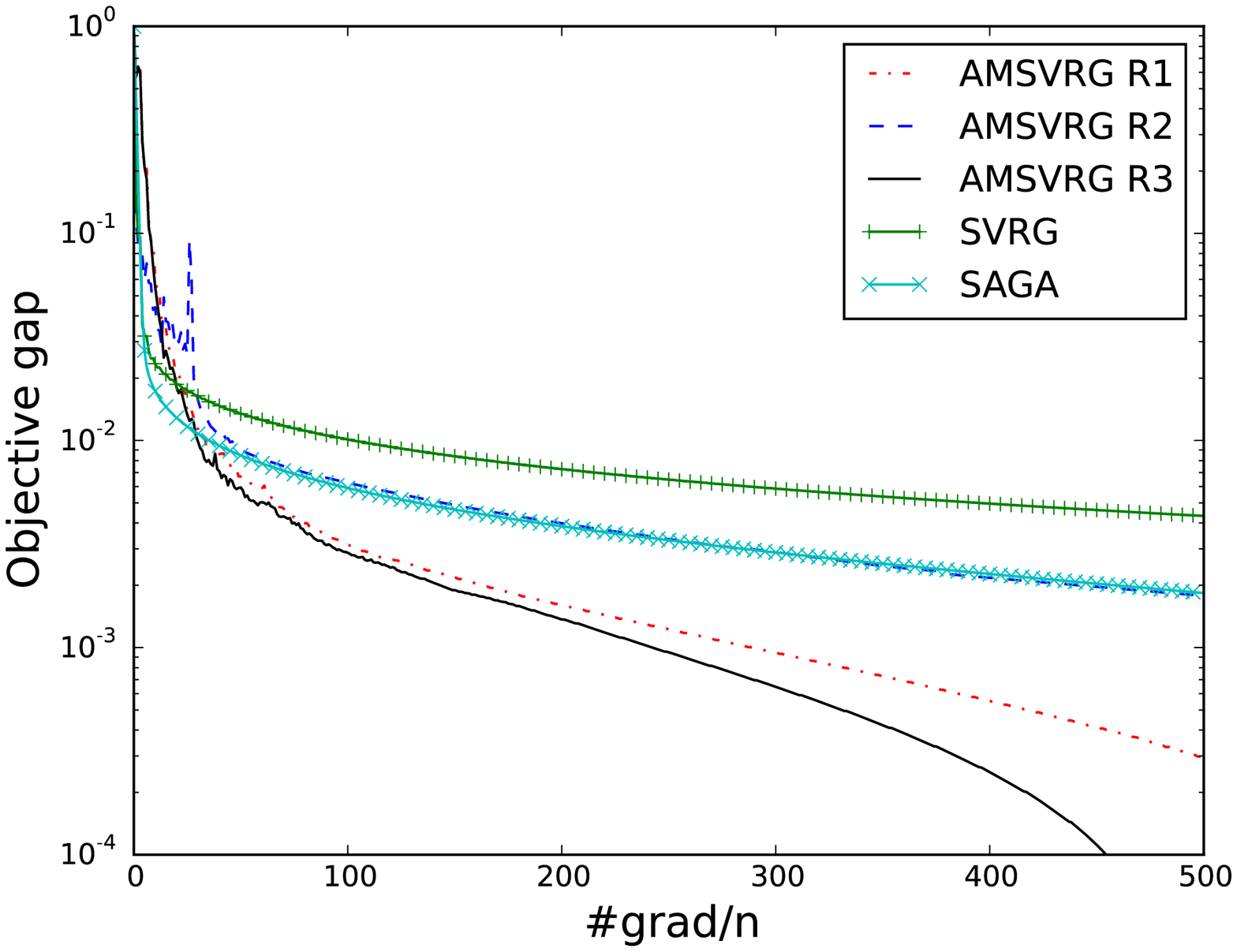} } }&
\hspace{-7mm}  \multirow{8}{*}{ \resizebox{41mm}{!}{ \includegraphics[angle=0]{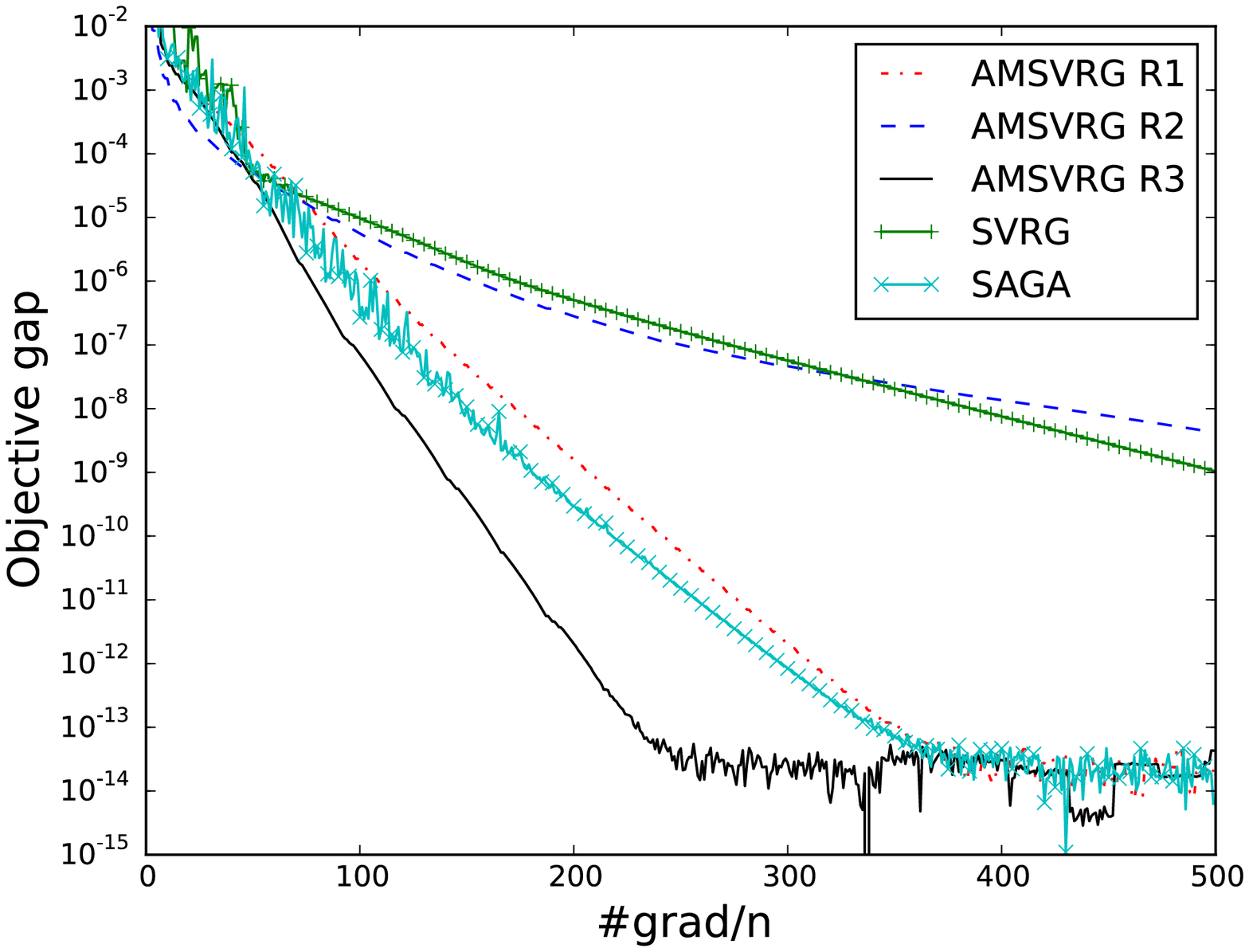} } }&
\hspace{-7mm} \multirow{8}{*}{ \resizebox{41mm}{!}{ \includegraphics[angle=0]{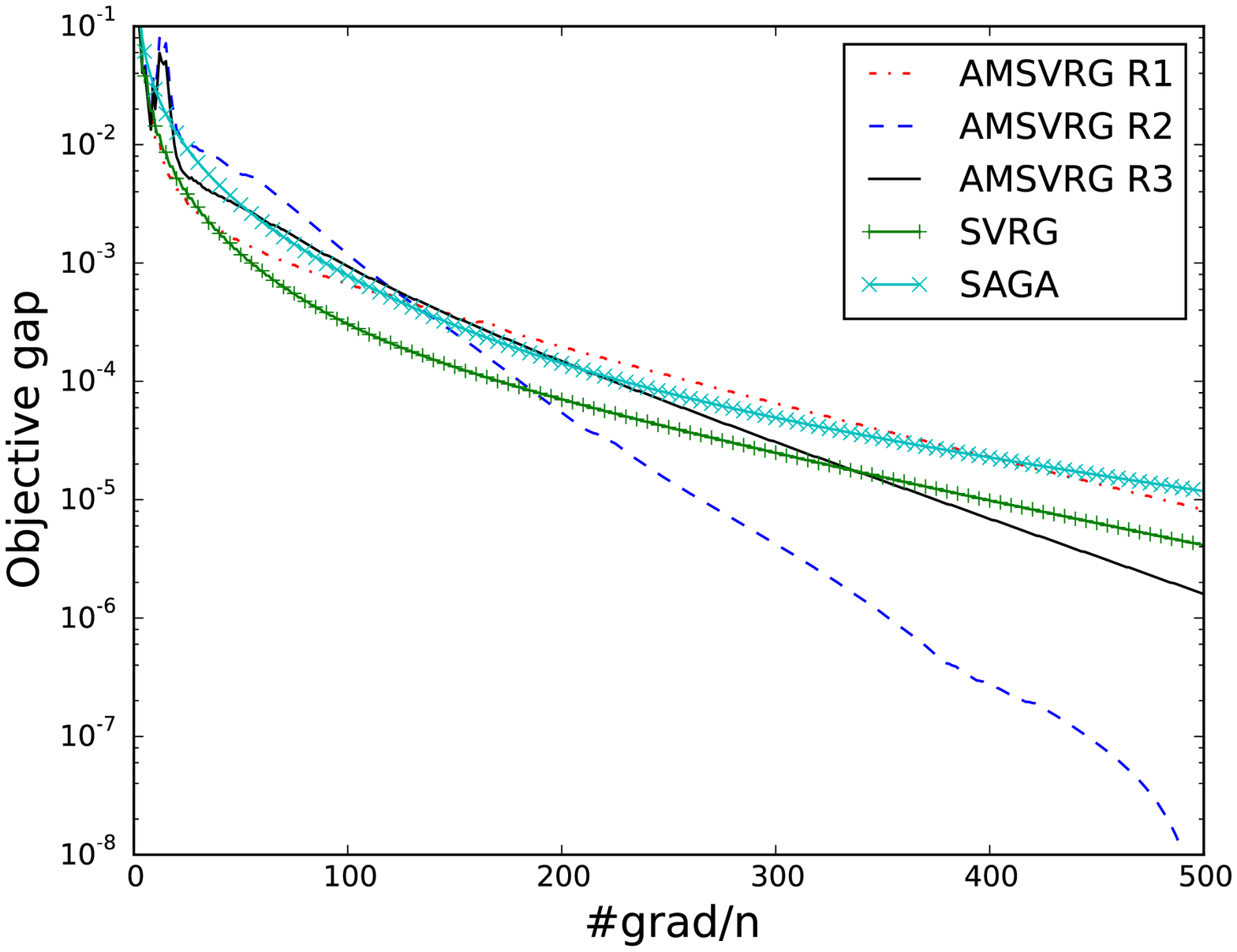} } }\\
&&&\\
&&&\\
$10^{-7}$ &&&\\
&&&\\
&&&\\
&&&\\
&&&\\
&
\hspace{-2mm}  \multirow{8}{*}{ \resizebox{41mm}{!}{ \includegraphics[angle=0]{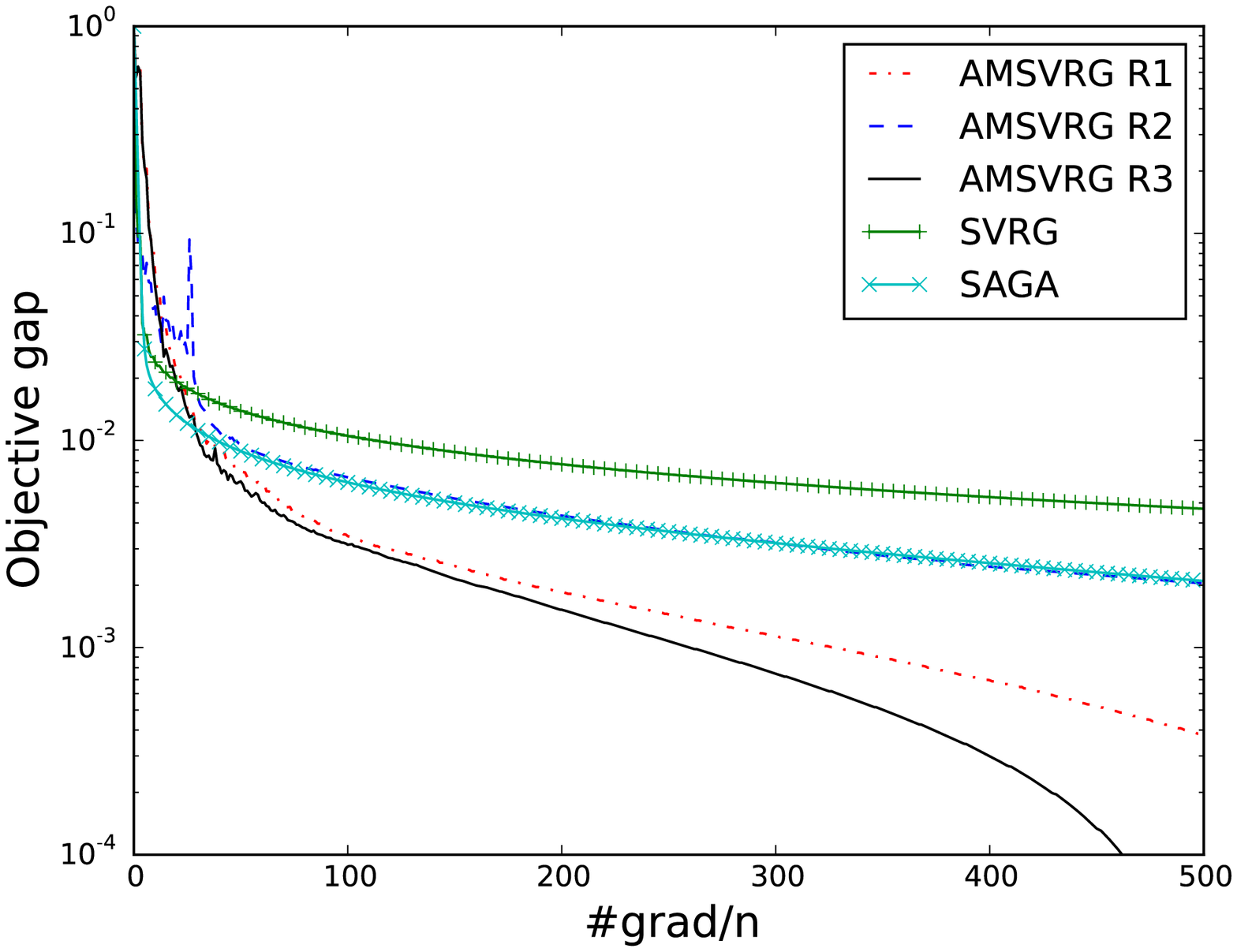} } }&
\hspace{-7mm}  \multirow{8}{*}{ \resizebox{41mm}{!}{ \includegraphics[angle=0]{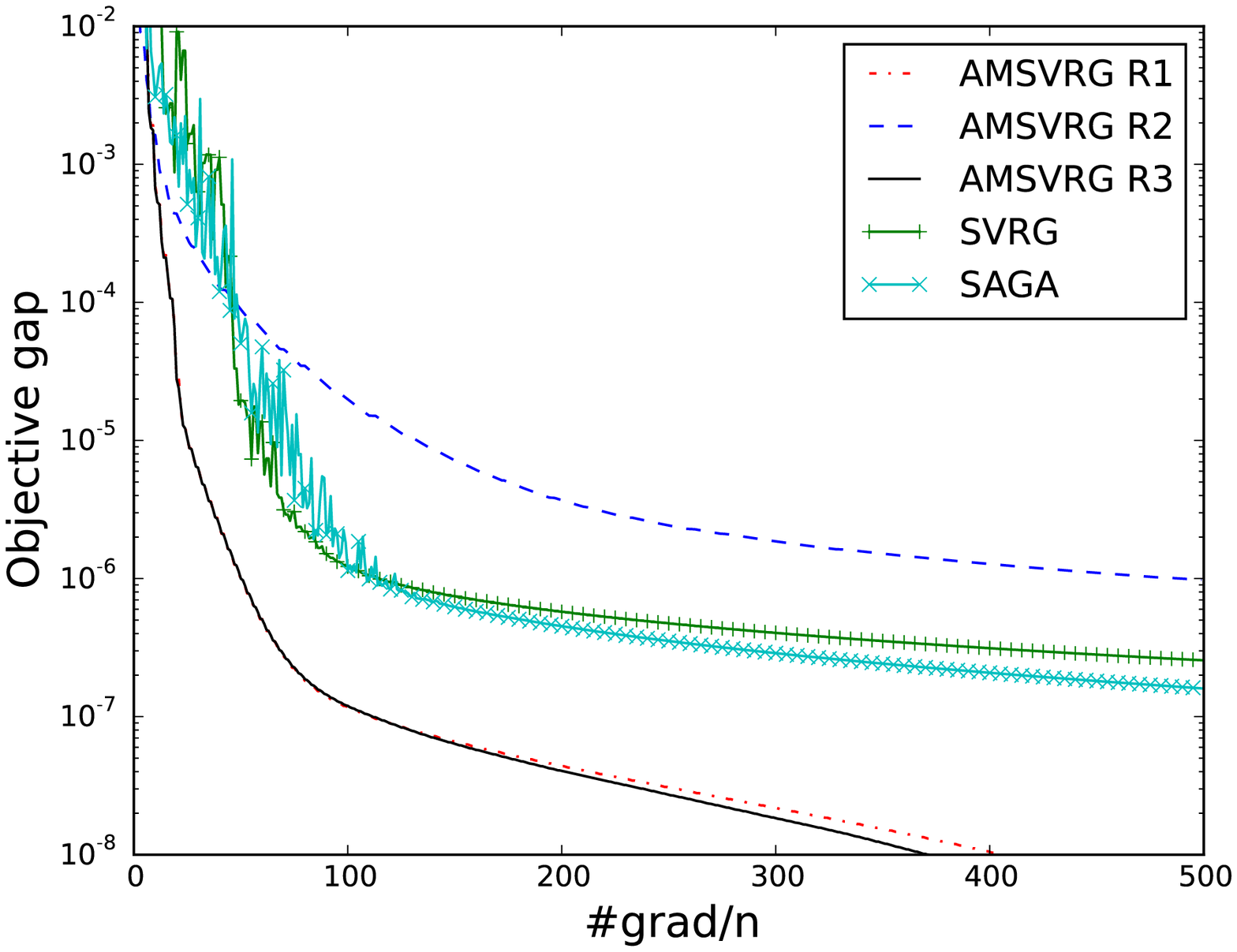} } }&
\hspace{-7mm} \multirow{8}{*}{ \resizebox{41mm}{!}{ \includegraphics[angle=0]{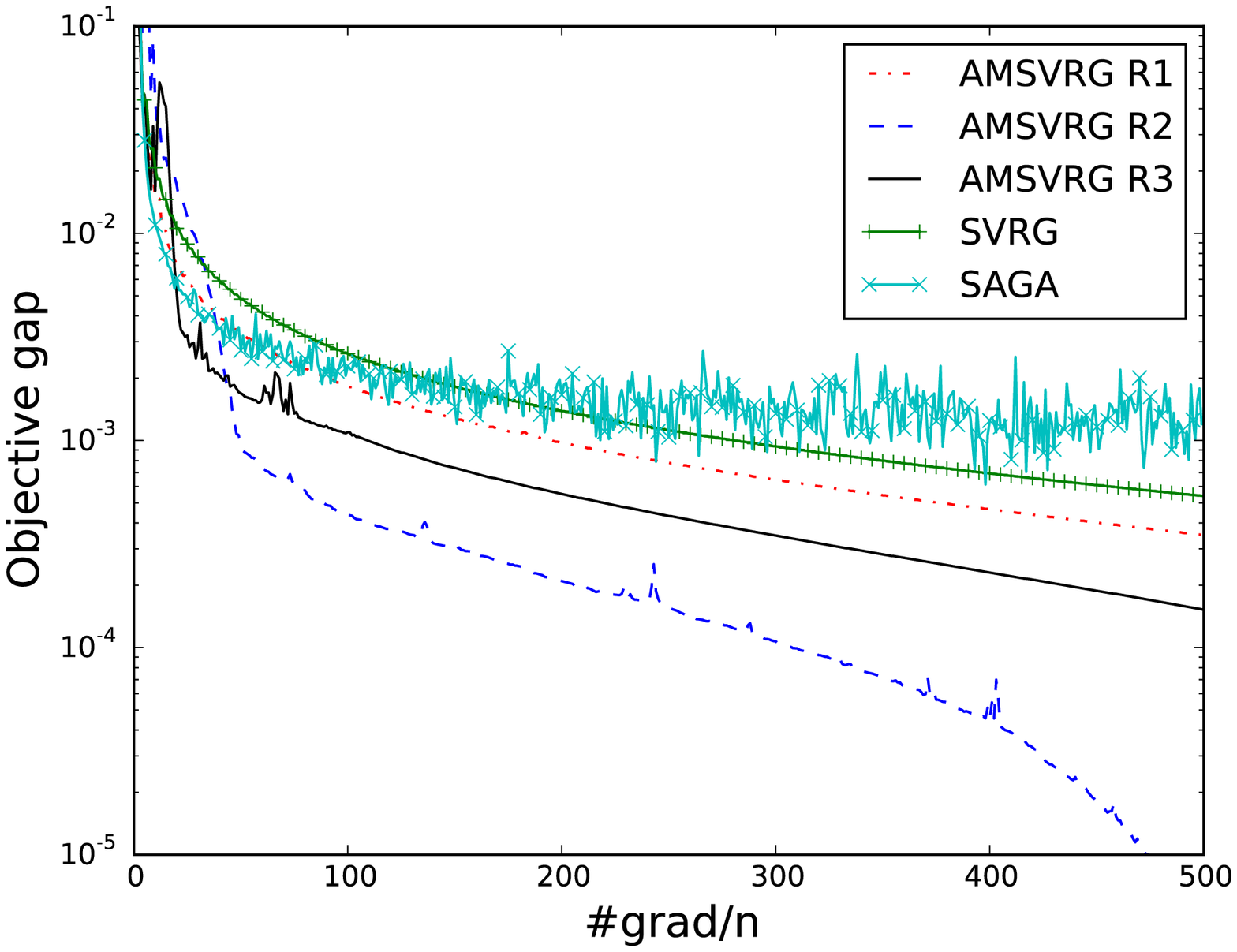} } }\\
&&&\\
&&&\\
$0$ &&&\\
&&&\\
&&&\\
&&&\\
&&&\\
\end{tabular}
\caption{ Comparison of algorithms applied to $L2$-regularized multi-class logistic regularization (left: mnist, middle: covtype), and $L2$-regularized binary-class logistic regularization (right: rcv1). }
\label{experiments}
\end{center}
\end{figure}

\section{Conclusion}
We propose method that incorporates acceleration gradient method and the SVRG in the increasing mini-batch setting.
We showed that our method achieves a fast convergence complexity for non-strongly and strongly convex problems.

\newpage
\begin{center}
\begin{LARGE}
{\bf Supplementary Materials}
\end{LARGE}
\end{center}
\setcounter{section}{0}
\renewcommand{\thesection}{\Alph{section}}
\section{Proof of the Lemma \ref{svrg_bound_lemma} }
To prove Lemma \ref{svrg_bound_lemma}, the following lemma is required, which is also shown in \cite{Fre1962}.
\begin{lemma_appendix}
Let $\{\xi_i\}_{i=1}^n$ be a set of vectors in $\mathbb{R}^d$ and $\mu$ denote an average of $\{\xi_i\}_{i=1}^n$.
Let $I$ denote a uniform random variable representing a size $b$ subset of $\{1,2,\ldots,n\}$.
Then, it follows that,
\[ \mathbb{E}_I \left \| \frac{1}{b}\sum_{i \in I}\xi_i - \mu \right \|^2 = \frac{n-b}{b(n-1)}\mathbb{E}_i \| \xi_i - \mu \|^2. \]
\end{lemma_appendix}
\begin{proof}
We denote a size $b$ subset of $\{1,2,\ldots,n\}$ by $S=\{ i_1,\ldots, i_b \}$ and denote $\xi_i-\mu$ by $\tilde{\xi}_i$. Then,
\begin{eqnarray*}
\mathbb{E}_I \left \| \frac{1}{b}\sum_{i \in I}\xi_i - \mu \right \|^2 
&=& \frac{1}{C(n,b)} \sum_{S} \left \| \frac{1}{b}\sum_{j=1}^b \xi_{i_j} - \mu \right \|^2  \\
&=& \frac{1}{b^2C(n,b)} \sum_S \left \| \sum_{j=1}^b \tilde{\xi}_{i_j} \right \|^2 \\
&=& \frac{1}{b^2C(n,b)} \sum_S \left( \sum_{j=1}^b \| \tilde{\xi}_{i_j}\|^2  + 2 \sum_{j,k,j< k} \tilde{\xi}_{i_j}^T\tilde{\xi}_{i_k} \right),
\end{eqnarray*}
where $C(\cdot,\cdot)$ is a combination. 
By symmetry, an each $\tilde{\xi}_i$ appears $\frac{ b C(n,b) }{ n }$ times 
and an each pair $\tilde{\xi}_i^T \tilde{\xi}_j$ for $i<j $ appears $\frac{ C(b,2)C(n,b) }{ C(n,2) } $ times in $\sum_S$.
Therefore, we have
\begin{eqnarray*}
\mathbb{E}_I \left \| \frac{1}{b}\sum_{i \in I}\tilde{\xi}_i - \mu \right \|^2 &=&
\frac{1}{b^2C(n,b)} \left( \frac{ b C(n,b) }{ n } \sum_{i=1}^n\| \tilde{\xi}_i\|^2 + \frac{ 2C(b,2)C(n,b) }{ C(n,2) } \sum_{i,j,i< j} \tilde{\xi}_i^T \tilde{\xi}_j  \right) \\
&=&\frac{1}{bn} \sum_{i=1}^n\| \tilde{\xi}_i\|^2 + \frac{ 2(b-1)}{ bn(n-1) } \sum_{i,j,i< j} \tilde{\xi}_i^T \tilde{\xi}_j.
\end{eqnarray*}
Since, $0=\| \sum_{i=1}^n \tilde{\xi}_i \|^2 = \sum_{i=1}^n \|\tilde{\xi}_i\|^2 + 2 \sum_{i,j,i< j} \tilde{\xi}_i^T \tilde{\xi}_j$, we have
\begin{eqnarray*}
\mathbb{E}_I \left \| \frac{1}{b}\sum_{i \in I}\tilde{\xi}_i - \mu \right \|^2 
= \left( \frac{1}{bn} - \frac{b-1}{bn(n-1)} \right) \sum_{i=1}^n\| \tilde{\xi}_i\|^2
= \frac{n-b}{b(n-1)} \frac{1}{n}\sum_{i=1}^n\| \tilde{\xi}_i\|^2.
\end{eqnarray*}
This finishes the proof of Lemma.
\end{proof}

We now prove the Lemma \ref{svrg_bound_lemma}.
\begin{proof}[ Proof of Lemma \ref{svrg_bound_lemma} ]
We set  $v_j^1 = \nabla f_j(x_k) - \nabla f_j(\tilde{x}) + \tilde{v}$.
Using Lemma A and 
\[ v_k = \frac{1}{b}\sum_{j\in I_k}v_j^1, \]
conditional variance of $v_k$ is as follows
\[ \mathbb{E}_{I_k}\| v_k - \nabla f(x_k) \|^2=\frac{1}{b}\frac{n-b}{n-1}\mathbb{E}_{j}\| v_j^1 - \nabla f(x_k) \|^2, \]
where expectation in right hand side is taken with respect to $j \in \{1,\ldots,n\}$.
By Corollary 3 in \cite{XZ2014_}, it follows that,
\begin{equation}
 \mathbb{E}_{j}\| v_j^1 - \nabla f(x_k) \|^2 \leq  4L(f(x_k)-f(x_*)+f(\tilde{x})-f(x_*) ).  \label{lem_clame} \nonumber
\end{equation}
This completes the proof of Lemma \ref{svrg_bound_lemma}.
\end{proof}

\section{Stochastic gradient descent analysis }
Below is the proof of Lemma \ref{gradient_descent_inequality}.
\begin{proof}[ Proof of Lemma \ref{gradient_descent_inequality} ]
It is clear that $y_k$ is equal to $x_k - \eta v_k$.
Since $f(x)$ is $L$-smooth and $\eta=\frac{1}{L}$, we have,
\begin{eqnarray*}
f(y_k)&\leq& f(x_k) + ( \nabla f(x_k), y_k - x_k) + \frac{L}{2} \| y_k-x_k\|^2 \\
&=& f(x_k) - \frac{1}{L}(\nabla f(x_k), v_k) + \frac{1}{2L}\| v_k \|^2.
\end{eqnarray*}
$v_k$ is an unbiased estimator of gradient $\nabla f(x_k)$, that is, $\mathbb{E}_{I_k}[v_k]=\nabla f(x_k)$.
Hence, we have
\[ \mathbb{E}_{I_k} \| v_k \|^2 = \|\nabla f(x_k) \|^2 + \mathbb{E}_{I_k} \|v_k - \nabla f(x_k)\|^2. \]
Using above two expressions, we get
\begin{eqnarray*}
\mathbb{E}_{I_k}[ f(y_k) ] &=& f(x_k) - \frac{1}{L} \| \nabla f(x_k) \|^2 + \frac{1}{2L} \mathbb{E}_{I_k} \| v_k \|^2 \\
&=& f(x_k) - \frac{1}{2L} \| \nabla f(x_k) \|^2 + \frac{1}{2L} \mathbb{E}_{I_k} \|v_k - \nabla f(x_k)\|^2.
\end{eqnarray*}
\end{proof}

\section{Stochastic mirror descent analysis}
We give the proof of Lemma \ref{mirror_descent_inequality}.
\begin{proof}[ Proof of Lemma \ref{mirror_descent_inequality} ]
The following are basic properties of Bregman divergence.
\begin{eqnarray}
& &(\nabla V_x(y), u - y) = V_x(u) - V_y(u) - V_x(y), \label{div_ineq_1} \\
& &V_x(y) \geq \frac{1}{2} \|x-y\|^2. \label{div_ineq_2}  
\end{eqnarray}
Using (\ref{div_ineq_1}) and (\ref{div_ineq_2}), we have
\begin{eqnarray*}
\alpha_k(v_k,z_{k-1}-u) &=& \alpha_k(v_k,z_{k-1}-z_k) + \alpha_k(v_k,z_k-u) \\
&=& \alpha_k(v_k,z_{k-1}-z_k) - ( \nabla V_{z_{k-1}}(z_k), z_k - u ) \\
&\underset{(\ref{div_ineq_1})}{=}& \alpha_k(v_k,z_{k-1}-z_k) + V_{z_{k-1}}(u) - V_{z_k}(u) - V_{z_{k-1}}(z_k) \\
&\underset{(\ref{div_ineq_2})}{\leq}& \alpha_k(v_k,z_{k-1}-z_k) - \frac{1}{2}\|z_{k-1}-z_k \|^2 + V_{z_{k-1}}(u) - V_{z_k}(u) \\
&\leq& \frac{1}{2} \alpha_k^2 \|v_k\|^2 + V_{z_{k-1}}(u) - V_{z_k}(u),
\end{eqnarray*}
where for the second equality we use stochastic mirror descent step, that is, $\alpha_kv_k + \nabla V_{z_{k-1}}(z_k)=0$ and 
for the last inequality we use the Fenchel-Young inequality $ \alpha_k(v_k,z_{k-1}-z_k) \leq \frac{1}{2}\alpha_k^2 \|v_k\|^2 + \frac{1}{2}\|z_{k-1}-z_k \|^2$. 

By taking expectation with respect to $I_k$ and using $\mathbb{E}_{I_k} \| v_k \|^2 = \|\nabla f(x_k) \|^2 + \mathbb{E}_{I_k} \|v_k - \nabla f(x_k)\|^2$, we have 
\begin{equation*}
\alpha_k(\nabla f(x_k),z_{k-1}-u) \leq V_{z_{k-1}}(u) - \mathbb{E}_{I_k} [ V_{z_k}(u) ] + \frac{1}{2}\alpha_k^2\|\nabla f(x_k)\|^2 + \frac{1}{2}\alpha_k^2\mathbb{E}_{I_k}\| v_k - \nabla f(x_k) \|^2.
\end{equation*}
This finishes the proof of Lemma \ref{mirror_descent_inequality}.
\end{proof}

\section{Modified AMSVRG for general convex problems}
We now intrdouce a modified AMSVRG (described in Figure \ref{multi_stage_mod}) that do not need the boundedness assumption for general convex problems.
\begin{figure}[h]
\begin{center}
\fbox{\rule{0cm}{0cm} 
\begin{tabular}{l}
{\bf Algorithm 3}$(w_0,\ (m_s)_{s\in \mathbb{Z}_+},\ \eta,\ (\alpha_{k+1})_{k\in \mathbb{Z}_+},\ (b_{k+1})_{k\in \mathbb{Z}_+},\ (\tau_k)_{k\in \mathbb{Z}_+} )$ 
\\ \hline \\
{\bf for} $s \leftarrow 0,\ 1,\ldots$ \\
\ \ \ \ \ \ $y_0 \leftarrow w_s,\ \ z_0 \leftarrow w_0$\\
\ \ \ \ \ \ $w_{s+1} \leftarrow {\bf Algorithm 1 }( y_0,\ z_0,\ m_s,\ \eta,\ (\alpha_{k+1})_{k\in \mathbb{Z}_+},\ (b_{k+1})_{k\in \mathbb{Z}_+},\ (\tau_k)_{k\in \mathbb{Z}_+})$\\
{\bf end} \\
\end{tabular}
\rule{0cm}{0cm}}
\end{center}
\caption{Modified AMSVRG}
\label{multi_stage_mod}
\end{figure}
We set $\eta, \alpha_{k+1},$ and $\tau_k$ as in (\ref{params}). 
Let $b_{k+1} \in \mathbb{Z}_+$ be the minimum values satisfying $4L \delta_{k+1}\alpha_{k+1} \leq p$ for small $p\ (e.g.\ 1/4)$.
Let $m_s = \left \lceil 4\sqrt{\frac{LV_{z_0}(x_*)}{\epsilon}} \right \rceil$.
From Thorem \ref{theorem1}, we get
\[ \mathbb{E}[ f(w_{s+1})  - f(x_*) ] \leq \epsilon + a( f(w_s) - f(x_*) ), \]
where $a = \frac{5}{2}p$.
Thus, it followis that,
\begin{eqnarray*}
 \mathbb{E}[ f(w_{s+1})  - f(x_*) ] &\leq& \sum_{t=0}^s a^t \epsilon + a^{s+1}( f(w_0) - f(x_*) )\\
&\leq& \frac{1}{1-a} \epsilon + a^{s+1}( f(w_0) - f(x_*) ).
\end{eqnarray*}
Hence, running the modified AMSVRG for $O\left(\log \frac{1}{\epsilon}\right)$ outer iterations achieves $\epsilon$-accurate solution in expectation,
and a complexity at each stage is
\begin{eqnarray*}
&&\ \ \ \ O\left( n + \sum_{k=0}^{m_s} b_{k+1} \right) \leq O\left( n + \frac{nm_s^2}{n+m_s} \right) \\
&&=O\left( n + \frac{nL}{\epsilon n + \sqrt{\epsilon L}} \right) = O\left( n + \min \left\{ \frac{L}{\epsilon}, n\sqrt{ \frac{L}{\epsilon}}\ \right\} \right),
\end{eqnarray*}
where we used the monotonicity of $b_{k+1}$ with respect to $k$ for the first inequality. 
Note that $V_{z_0}(x_*)$ is constant (i.e. $V_{w_0}(x_*)$), and $O$ hides this term.
From the above analysis, we derive the following theorem.
\begin{theorem} 
Consider the modified AMSVRG under Assumptions \ref{assumption_l_smooth}.
Let parameters be as above.
Then the overall complexity for obtaining $\epsilon$-accurate solution in expectation is 
\[ O\left( \left( n + \min \left\{ \frac{L}{\epsilon}, n\sqrt{ \frac{L}{\epsilon}}\ \right\} \right) \log \left(\frac{1}{\epsilon} \right) \right). \]
\end{theorem}


\begin{thebibliography}{9}
  \bibitem{RSB2012} N. Le Roux, M. Schmidt, and F. Bach. A stochastic gradient method with an exponential convergence rate for 
      finite training sets. {\it Advances in Neural Information Processing System 25}, pages 2672-2680, 2012.

  \bibitem{SRB2013} M. Schmidt, N. Le Roux, and F. Bach. Minimizing finite sums with the stochastic average gradient. {\it arXiv:1309.2388}, 2013.

  \bibitem{SZ2012} S. Shalev-Shwartz and T. Zhang. Proximal stochastic dual coordinate ascent. {\it arXiv:1211.2717}, 2012.

  \bibitem{SZ2013a} S. Shalev-Shwartz and T. Zhang. Stochastic dual coordinate ascent methods for regularized loss minimization. 
    {\it Journal of Machine Learning Research 14}, pages 567-599, 2013.

  \bibitem{JZ2013} R. Johnson and T. Zhang. Accelerating stochastic gradient descent using predictive variance reduction. 
    {\it Advances in Neural Information Processing System 26}, pages 315-323, 2013.

  \bibitem{KR2013} J. Kone\v{c}n\'y and P. Richt\'arik. Semi-stochastic gradient descent methods. {\it arXiv:1312.1666}, 2013.

  \bibitem{SZ2014} S. Shalev-Shwartz and T. Zhang. Accelerated proximal stochastic dual coordinate ascent for regularized loss minimization.
    {\it Proceedings of the 31th International Conference on Machine Learning}, pages 64-72, 2014.

  \bibitem{XZ2014} L. Xiao and T. Zhang. A proximal stochastic gradient method with progressive variance reduction. {\it arXiv:1403.4699}, 2014.

  \bibitem{M2015} J. Mairal. Incremental majorization-minimization optimization with application to large-scale machine learning. {\it SIAM Journal on Optimization}, 25(2), pages 829-855, 2015.

  \bibitem{DBL2014} A. Defazio, F. Bach, and S. Lacoste-Julien. SAGA: A fast incremental gradient method with support for non-strongly convex composite objectives.
    {\it Advances in Neural Information Processing System 27}, pages 1646-1654, 2014.

  \bibitem{Nit2014} A. Nitanda. Stochastic proximal gradient descent with acceleration techniques. 
    {\it Advances in Neural Information Processing System 27}, pages 1574-1582, 2014.

  \bibitem{KLR2015} J. Kone\v{c}n\'y, J. Lu, and P. Richt\'arik. Mini-batch semi-stochastic gradient descent in the proximal setting. {\it arXiv:1504.04407}, 2015.

  \bibitem{AZO2015} Z. Allen-Zhu and L. Orecchia. Linear coupling of gradient and mirror descent: A novel, simple interpretation of Nesterov's accelerated method. {\it arXiv:1407.1537}, 2015.

  \bibitem{Nes2005} Y. Nesterov. Smooth minimization of non-smooth functions. {\it Mathematical Programming}, 103(1), pages 127-152, 2005.

  \bibitem{Nes2004} Y. Nesterov.  {\it Introductory Lectures on Convex Optimization: A Basic Course}. Kluwer, Boston, 2004.

  \bibitem{BL2005} L. Bottou and Y. LeCun. On-line learning for very large datasets. {\it Applied Stochastic Models in Business and Industry}, 21(2), pages 137-151, 2005.

  \bibitem{GOP2014} M. G\"urb\"uzbalaban, A. Ozdaglar, and P. Parrilo. A globally convergent incremental Newton method. {\it arXiv:1410.5284}, 2014.

  \bibitem{DC2013} B. O'Donoghue and E. Cand\'es. Adaptive restart for accelerated gradient schemes. {\it Foundations of Computational Mathematics}, pages 1-18, 2013.

  \bibitem{GB2014} P. Giselsson and S. Boyd. Monotonicity and restart in fast gradient methods. {\it In 53rd IEEE Conference on Decision and Control}, pages 5058-5063, 2014.

  \bibitem{SBC2014} W. Su, S. Boyd, and E. Cand\'es. A differential equation for modeling Nesterov's accelerated gradient method: theory and insights. {\it Advances in Neural Information Processing System 27}, pages 2510-2518, 2014.

  \bibitem{AD2011} A. Agarwal and J. Duchi. Distributed delayed stochastic optimization.
    {\it Advances in Neural Information Processing System 24}, pages 873-881, 2011.

  \bibitem{DGSX2012} O. Dekel, R. Gilad-Bachrach, O. Shamir, and L. Xiao. Optimal distributed online prediction using mini-batches.
    {\it Journal of Machine Learning Research 13}, pages 165-202, 2012.

  \bibitem{SZ2013b} S. Shalev-Shwartz and T. Zhang. Accelerated mini-batch stochastic dual coordinate ascent.
    {\it Advances in Neural Information Processing System 26}, pages 378-385, 2013.

\end{thebibliography}

\begin{thebibliography}{9}
  \bibitem{Fre1962} J. E. Freund. {\it Mathematical Statistics}. prentice Hall, 1962.
  \bibitem{XZ2014_} L. Xiao and T. Zhang. A proximal stochastic gradient method with progressive variance reduction. {\it arXiv:1403.4699}, 2014.
\end{thebibliography}
\end{document}